\documentclass[lettersize,journal]{IEEEtran}
\usepackage{amsthm, amsmath,amsfonts}
\usepackage{algorithmic}
\usepackage{algorithm}
\usepackage{array}
\usepackage[caption=false,font=normalsize,labelfont=sf,textfont=sf]{subfig}
\usepackage{textcomp}
\usepackage{stfloats}
\usepackage{url}
\usepackage{verbatim}
\usepackage{graphicx}
\usepackage{cite}
\usepackage{balance}
\newtheorem{theorem}{Theorem}[section]
\newtheorem{lemma}[theorem]{Lemma}

\newtheorem{corollary}{Corollary}
\newtheorem{definition}{Definition}[section]

\begin{document}

\title{Subsethood Measures of Spatial Granules}

\author{Liquan Zhao, Yiyu Yao
\thanks{Liqan Zhao and Yiyu Yao are with the Department of Computer Science, University of Regina, Regina, Saskatchewan S4S 0A2, Canada (e-mail: bushzhao@gmail.com; yiyu.yao@uregina.ca)}}

\markboth{}%
{Shell \MakeLowercase{\textit{et al.}}: A Sample Article Using IEEEtran.cls for IEEE Journals}


\maketitle

\begin{abstract}
Subsethood, which is to measure the degree of set inclusion relation, is predominant in fuzzy set theory. This paper introduces some basic concepts of spatial granules, coarse-fine relation, and operations like meet, join, quotient meet and quotient join. All the atomic granules can be hierarchized by set-inclusion relation and all the granules can be hierarchized by coarse-fine relation. Viewing an information system from the micro and the macro perspectives, we can get a micro knowledge space and a micro knowledge space, from which a rough set model and a spatial rough granule model are respectively obtained. The classical rough set model is the special case of the rough set model induced from the micro knowledge space, while the spatial rough granule model will be play a pivotal role in the problem-solving of structures. We discuss twelve axioms of monotone increasing subsethood and twelve corresponding axioms of monotone decreasing supsethood, and generalize subsethood and supsethood to conditional granularity and conditional fineness respectively. We develop five conditional granularity measures and five conditional fineness measures and prove that each conditional granularity or fineness measure satisfies its corresponding twelve axioms although its subsethood or supsethood measure only hold one of the two boundary conditions. We further define five conditional granularity entropies and five conditional fineness entropies respectively, and each entropy only satisfies part of the boundary conditions but all the ten monotone conditions.  
\end{abstract}


\section{Introduction}
\IEEEPARstart{S}UBSETHOOD was first used to measure fuzzy sets, and it is denoted by a bivalent function to show the degree of a fuzzy set being a subset of another fuzzy set ~\cite{dubois1980, goguen1969, willmott1980, willmott1986, kosko1986}. Kosko \cite{kosko1986,kosko1990,kosko1991,kosko1996} generalized this concept and defined a multivalent subsethood measure. Subsethood has drawn the attention of many scholars who related subsethood with entropy ~\cite{kosko1986,bustince2006,bustince2008,vlachos2007,young1996}, distance measure ~\cite{vlachos2007,grzegorzewski2004,zhangz2009}, similarity measure ~\cite{zhangz2009, fan1999, huang2005, li2013} and logical implication ~\cite{bandler1978, bandler1980, burillo2000, fan1999a,ragin1987,wierman2007}. Most of subsethood studies focus on fuzzy sets and there are only a few of them in rough sets. What's more, these studies mainly discussed the desired properties of subsethood measures or weak subsethood measures and paid little attention to the construction of specific measures. Yao and Deng \cite{yao2014} constructed subsethood measures of two sets based on two views: one is different equivalent expressions of the condition $A \subseteq B$ and the other is the grouping of objects based on two sets $A$ and $B$. When applying subsethood to rough sets, it shows the graded set-inclusion relation of different sets, they are quantitative generalizations of the set-inclusion relation and can be used to distinguish those sets with same size in some degree. 

A partition is the simplest granulation scheme and hence measurement of partitions has been proposed and studied. Yao and Zhao \cite{yao2012b} divide these measures into two classes: information-theoretic measures and interaction-based measures. Hartley entropy and Shannon entropy are typical representatives of information-theoretic measures. Although Hartely entropy coincides with the Shannon entropy in the case of a uniform probability distribution, Klir and Golger \cite{klir1988} pointed out they are semantic differences. Shannon entropy is a measure of information induced by a probability distribution while Hartley entropy is a measure of nonspecificity of a finite set. Their uses as measures of the granularity of partitions were suggested and examined in \cite{yao2012b, klir1988, beaubouef1998,duntsch1998,duntsch2001,lee1987,liang2002,liang2004,liang2006,liang2009,miao1998,miao1999,qian2008,qian2009,wang2008,wierman1999,yao2003a,zhu2012}. Interaction based measures count the number of interacting pairs of elements of a universal set under a partition. Each pair in the equivalence relation is counted as one interaction, and the size of the equivalence relation denotes the total number of interactions. Miao and Fan \cite{miao2002} first defined an interaction based measure of granularity of a partition which may be interpreted as a normalized cardinality of an equivalence relation. Many authors studied this measure and extended it \cite{yao2012b,liang2002,liang2004,liang2006,liang2009,wang2008,xu2005}. However, the extensions mainly focus on non-equivalence relations. 

Granular computing (GrC) is not an algorithm or process but an idea, and, in fact, this idea has been permeated through every computing theory since the very beginning. The definition or construction of information granules is one of the basic issues of GrC. By Merriam-webster dictionary, the word ``granule'' has two meanings: one is a small particle, and the other is one of numerous particles forming a larger unit. People generally choose its first meaning, that is, a granule is defined as a simple crisp or fuzzy set. Zhao \cite{zhao2007,zhao2008} first introduced its second meaning as the general definition of granules, and extended the partitions to equivalence granules and the finite set to infinite set as well. He think a granule is made up of one or more atomic granules, which are indivisible under the giving subdivision rule. However, these atomic granules may be divisible under its finer subdivision rules, that is to say whether an atomic granule is divisible or not is relative. There are structural and nonstructural relationships between the atomic granules. This is a structural definition which can show the spatiality of a granule, and the granules defined by this way is called the spatial granules so as to distinguish from the granules defined by the previous way. 

The contribution and organization of this paper is organized as follows: 

In Section \ref{section_model}, we introduce the basic notions of granules, coarse-fine relation, which is the generalization of set-inclusion relation, and operations like meet, join, quotient meet and quotient join, which are generalizations of intersection or union. All the atomic granules can be hierarchized by set-inclusion relation, and all the granules can be hierarchized by coarse-fine relation. Given an information system, when performing the micro and macro granular analysis on it, we can generate a micro knowledge space and a macro knowledge space, from which a rough set model and a spatial rough granule model are respectively induced. The rough set model can be used for incomplete and complete information systems on any domain, and the classical rough set model is the special case of this one. The coarse-fine relation is the key to the success of hierarchical machine learning algorithms, and the spatial rough granule model will play a very important role in the structure problem solving. All the atomic granules can be hierarchized in a plane by set inclusion relation, and all granules can be hierarchized in an $n$-dimensional space by coarse-fine relation. 

In Section \ref{section_subsethood}, we discuss twelve properties of monotonically increasing subsethood and twelve corresponding properties of monotonically decreasing subsethood not only for atomic granules but also for granules, and the properties can be divided into two classes: boundary conditions and monotone conditions. The five monotonically increasing subsethood measures satisfy only one of the two boundary conditions but all ten monotone conditions. We construct five monotonically decreasing subsethood measures for atomic granules, and each one satisfies one or both the boundary conditions and ten monotonically decreasing conditions. Conditional granularity and conditional fineness are introduced to measure the coarse-fine relation between two granules. Conditional granularity is defined as the expectation of monotonically increasing subsethood of atomic granules with respect to the probability distribution of the meet of the two granules, and conditional fineness is defined as the expectation of monotonically decreasing subsethood of atomic granules with respect to the probability distribution of the meet of the two granules. We construct five conditional granularity measures and five conditional fineness measures and prove that each measure satisfies its corresponding twelve properties. Conditional granularity entropy and conditional fineness entropy are defined by their corresponding subsethood and the probability distribution of the meet of the two granules, where the five conditional granularity entropies satisfy part of the boundary conditions and ten monotonically increasing conditions and the five conditional fineness entropies satisfy part of the boundary conditions and ten monotonically decreasing conditions. 

\section{A Model of Spatial Granules}
\label{section_model}
\subsection{Preliminaries}

Given a universe of discourse $X=\{x_1,\cdots,x_n\}$, the granules and binary relations on $X$ are one-to-one corresponding, where the granules corresponding to fuzzy equivalence relations are called fuzzy equivalence granules and the granules corresponding to equivalence relations are called equivalence granules. Each equivalence granule is a partition of a subset of $X$, and, in particular, a partitions of $X$ is also called a quotient granule on $X$. For the sake of simplicity, we only discuss equivalence granules in this paper, that is, the atomic granules of a granule are its equivalence classes. 

Assume $A$ and $B$ are two subsets of $X, R_A$ and $R_B$ are equivalence relations on $A$ and $B$ respectively, and the equivalence granules corresponding to $R_A$ and $R_B$ are $A_R = \{a_1, \cdots, a_k\}$ and $B_R = \{b_1, \cdots, b_l\}$ respectively. For convenience, $A_R$ can also be denoted by $A$, and use granule $A$ or set $A$ to  to distinguish them so as not to cause ambiguity, that is, the granule $A$ is a partition of the set $A$. The operations of meet, join, quotient meet and quotient join are respectively defined as follow:
\begin{definition}
	\label{def_quotient-meet-join}
	\begin{enumerate}
		\item[]
		\item $A \wedge B$ is called the meet of $A$ and $B$, which is the granule corresponding to $R_A \cap R_B$; 
		\item $A \vee B$ is called the join of $A$ and $B$, which is the granule corresponding to $R_A \cup R_B$;
		\item $A \wedge_t B$ is called the quotient meet of $A$ and $B$, which is the granule corresponding to $t(R_A \cap R_B)$, the transitive closure of $R_A \cap R_B$; 
		\item $A \vee_t B$ is called the quotient join of $A$ and $B$, which is the granule corresponding to $t(R_A \cup R_B)$, the transitive closure of $R_A \cup R_B$.
	\end{enumerate}
\end{definition}
Where the quotient meet and quotient join operations are for (fuzzy) equivalence granules while the meet and join operations are for other granules. Obviously, for equivalence granules, the quotient meet is the same with meet but the join and quotient join are different. 

\subsection{Rough Set Model in Micro Knowledge Space}

Given an information system $I=(X,\textit{\textbf{R}})$, where $\textit{\textbf{R}}=\{ R_1,\cdots, R_m\}$ is a family of equivalence relations on subsets of $X=\{x_1,\cdots,x_n\}$. This information system can be viewed from the micro and the macro perspectives respectively. From the micro perspective, we think about all the subsets of $X$, denoted as $\sigma(X)$. $(\sigma(X), \supseteq)$ is a complete lattice, and all the elements in $\sigma(X)$ can be hierarchized under set inclusion relation. 

Assume the equivalence granules corresponding to $ R_i$ are $P_i (i=1,\cdots,m)$, respectively, $R$ is the intersection of all $R_i(i=1,\cdots,m)$, and $P$ is the quotient meet of all $P_i(i=1,\cdots,m)$. For any $A \in \sigma(X), A$ is called $ R $-definable if it is one of the equivalence classes in $ P $ or a union of two or more equivalence classes in $P$. Assume $d(\sigma(X))$ is a family of all definable sets in $\sigma(X)$ and $d_0(\sigma(X))$ is a family of the empty set and all definable sets. Then $(d_0(\sigma(X)), \supseteq)$ is a complete bounded sublattice of $(\sigma(X), \supseteq)$, and $d_0(\sigma(X))$, which is closed under under union and intersection operations, is called the micro knowledge space generated from $I=(X,\textit{\textbf{R}})$. Therefore, $\sigma(X)$ can be divided into two categories: $d(\sigma(X))$ and $\widetilde{d}(\sigma(X))$, i.e., the family of all undefinable sets. By rough set theory, $\widetilde{d}(\sigma(X))$ can be further divided into $d_r(\sigma(X))$, i.e., the set of roughly definable sets, and $\widetilde{d}_r(\sigma(X))$, i.e., the set of roughly or totally undefinable sets. 
\begin{definition}
	For any $A \in \sigma(X)$, the lower and upper approximations of $A$ with respect to $R$ can be defined as: for every $B \in d(\sigma(X))$,
	\begin{align}
		\underline{R}(A) &= \bigcup \{ A \cap B \mid A \supseteq B\}, \nonumber \\
		\overline{R}(A)  &= \bigcap \{ A \cup B \mid B \supseteq A\}.
		\label{eqs_lower-upper-atom}
	\end{align}
\end{definition} 
Obviously, for any $A \in \sigma(X)$, its upper approximation is to find its least upper bound in $d_0(\sigma(X))$, and its lower approximation is to find the greatest lower bound in $d_0(\sigma(X))$. $(\underline{R}(A), \overline{R}(A))$ is called an approximation space of $A$.

When $I$ is complete, i.e., all $R_i(i=1,\cdots,n)$ are equivalence relations on $X$, every atomic granule in $d(\sigma(X))$ can be obtained from the atomic granules of $P$, and then we can replace $d(\sigma(X))$ with $P$. However, we should examine two extreme cases: $\forall B \in P, A \supset B$ and $\forall B \in P, B \supset A$. We can define its upper approximation as $ A $ in the first case and define its lower approximation as $ A $ in the second case. When $I$ is incomplete, not all of the atomic granules in $d(\sigma(X))$ can be obtained from the atomic granules of $P$. Therefore, we cannot replace $d(\sigma(X))$ with $P$. It can be seen that the classical rough set model is only for complete information systems while the above model is not only for complete information systems but also for incomplete information systems. When $ X $ a domain, we can divide it into $ n $ subdomains, which can be regarded as $n$ objects, and the above model is also applicable. All the extended models developed from the classical rough set model can be accordingly defined by $d(\sigma(X))$ so as to be applicable to any information system, which will be discussed in another paper.

\subsection{Rough granule Model in Macro Knowledge Space}

Assume $\Pi(\sigma(X))$ is the family of all equivalence granules on $X$ and $\Pi_0(\sigma(X))$ is the family of the empty granule and all equivalence granules on $X$. Then viewing $I$ from the macro perspective, the whole space is $\Pi_0(\sigma(X))$. There is no set inclusion relation between two granules, and we must define new relation.
\begin{definition}
	For any two equivalence relations $R_A, R_B$ over subsets of $X$, assume that their corresponding equivalence granules are $A$ and $B$, respectively. 
	\begin{enumerate}
		\item If $x, y \in X, xR_Ay \rightarrow xR_By,$ then $B$ is coarser than $A$ (or $A$ is finer than $B$), denoted by $B \succeq A$ (or $A \preceq B$);
		\item If $B \succeq A$ and $R_A \subset R_B$, then  $B$ is strictly coarser than $A$ (or $A$ is strictly finer than $B$), denoted by $B \succ A$ (or $A \prec B$);
		\item If $B \succeq A$ and $A \succeq B$, then two granules $A$ and $B$ are equal, denoted by $A=B$.	
	\end{enumerate} 
\end{definition}
$(\Pi_0(\sigma(X)), \succeq)$ is a complete bounded lattice \cite{zhao2007}, all the elements in $\Pi_0(\sigma(X))$ and the vertices of the unit $n$-dimensional hypercube are one-to-one corresponding, and $\Pi_0(\sigma(X))$ can be hierarchized by coarse-fine relation. For any granule $A \in \Pi(\sigma(X))$, it is called $ R $-definable under this information system if $A \succ P$. Assume $d(\Pi(\sigma(X)))$ is a family of all definable granules in $\Pi(\sigma(X))$ and $d_0(\Pi(\sigma(X)))$ is a family of $P$ and all definable granules.  Then $(d_0(\Pi(\sigma(X))), \succeq)$ is a complete bounded sublattice of $(\Pi_0(\sigma(X)), \succeq)$, and $d_0(\Pi(\sigma(X)))$, which is closed under under quotient meet and quotient join operations, is called the macro knowledge space generated from $I$. Therefore, $\Pi_0(\sigma(X))$ can be divided into two categories: $d(\Pi(\sigma(X)))$ and $\widetilde{d}(\Pi_0(\sigma(X)))$, i.e., the family of all undefinable granules. While $\widetilde{d}(\Pi_0(\sigma(X)))$ can be further divided into $d_r(\Pi(\sigma(X)))$, i.e., the set of roughly definable granules, and $\widetilde{d}_r(\Pi_0(\sigma(X)))$, i.e., the set of roughly or totally undefinable granules. 

For any granule $A$ in $\Pi(\sigma(X))$, its upper approximation is to find its lowest upper bound in $d_0(\Pi(\sigma(X)))$, and its lower approximation is to find the greatest lower bound in it. 

\begin{definition}
	The upper and lower approximations of granule $A$ with respect to $R$ can be defined as follows: for every $B \in d(\Pi(\sigma(X)))$
	\begin{align}
		\underline{R}(A) &= \bigvee\nolimits_t  \{A \wedge_t B \mid A \succeq B \}, \nonumber \\
		\overline{R}(A)  &= \bigwedge\nolimits_t \{A \vee_t B \mid B \succeq A\}.
	\end{align}
\end{definition}

The upper and lower approximations in the above model are not obtained from one of its tangent planes but from the $n$-dimensional space. Therefore, the model is also called the spatial rough granule model which can be applied to any structural information system and non-structural information system as well. In particular, we have $\underline{R}(A)=A \wedge_t P$ and $\overline{R}(A)=A \vee_t P$ when $ I $ is complete.

\section{Subsethood Measures of Two Granules}
\label{section_subsethood}

Measurement is the most important foundation of all computational theories and measurement of information granules is naturally the keystone of granular computing. Many measures of information granules have been discussed in different areas in isolation, and most of them focus on the measures of sets. We divide the measures into two classes: granularity or coarseness and fineness, where granularity is to measure the coarse degree of a granule and fineness is to measure the fineness degree of a granule \cite{zhao2007,zhao2017}. People mainly discuss granularity, to the extent that many people confuse the concepts of granularity and granule, and, in fact, entropy is a kind of fineness. Measurement of granules is not just to know the granularity or the fineness of each granule, but to know the coarse-fine relation, similarity and difference between two granules. The conditional granularity and conditional fineness defined in \cite{zhao2007, zhao2017} are to show the coarse-fine relation to some degree between two granules, and conditional granularity and conditional fineness clearly reflects the monotonically increasing and the monotonically decreasing respectively. While subsethood, in general, discuss monotonically increasing. In \cite{zhao2017}, we also show the conditional granularity is a generalization of subsethood measure, and it holds the axiomatic properties of subsethood measures that Yao and Deng discussed in \cite{yao2014}. Conditional granularity and conditional fineness are named from the point of view of probability distribution, while subsethood is named from the point of view of set inclusion. We can extend subsethood function to discuss monotonically decreasing so as to be generalized to denote conditional fineness. We can use any one to express the coarse-fine relation.        

\subsection{Subsethood of Two Atomic Granules}

Subsethood measures should satisfy some axioms to make them to be meaningful. Sinha and Dougherty \cite{sinha1993} presented nine axioms for subsethood and the last five ones further restrict subsethood measures, and Young \cite{young1996} mainly discussed the first four. Different scholars may define different axioms in different fields \cite{bustince2006, fan1999, sahin2015, vlachos2007, yao2014}. However, we can divided these axioms into two classes: basic axioms and extended axioms. Basic axioms are similar, and extended axioms may be different by the properties of empirical objects. 

In many situations, it is more convenient to consider a normalized measure for which the maximum value is $1$ and the minimum is $0$. For any two atomic granules $a, b \in \sigma(X)$, the basic axioms of a subsethood measure should satisfy: a subsethood measure must reach the maximum value if and only if $a \subseteq b$, it reaches the minimum value if and only if $a\cap b = \emptyset$, and it belongs to $[0,1]$; it should show the monotonicity because the set inclusion is a partial order relation.

\begin{definition}
	\label{def_subsethood-atoms}
	For any atomic granules $a, b \in \sigma(X)$, a function $sh: \sigma(X) \times \sigma(X) \longrightarrow {[}0,1{]}$ is called a normalized measure of subsethood if it satisfies the following two axioms (boundary conditions):
	\begin{description}      
		\item [(A1)] $sh(b, a) = 1 \Longleftrightarrow a \subseteq b;$
		\item [(A2)] $sh(b, a) = 0 \Longleftrightarrow a \cap b = \emptyset,$
	\end{description}
	where the value $sh(b, a)$ is the degree of $a$ being a subset of $b$.
\end{definition}

For the classical set inclusion, a set $a$ is either a subset of another set $b$ or not, i.e., $sh(b, a)$ is either 1 or 0, and the conditions (A1) and (A2) are dual each other. Some authors~\cite{xu2002,zhangl1996a} used a single implication:
\[a \subseteq b \Longrightarrow sh(b, a) = 1.\]
That is, $sh(b, a)$ reaches the maximum value if $a\subseteq b$. However we may still have $sh(b, a)=1$ even though $\neg (a\subseteq b)$. Gomoli\'{n}ska~\cite{gomolinska2008,gomolinska2014} considered the other single implication:
\[sh(b, a) =1 \Longrightarrow  a \subseteq b.\]

In this case, we can get $a\subseteq b$ from $sh(b, a) = 1$, and the other way around is not true. None of the two single implications can faithfully reflect whether a set is a subset of another besides the double implication.

For the general set inclusion, one set can be a subset of another one to some degree, that is, $sh(b, a)$, the degree of the inclusion, can be any value between 0 and 1. When researching on subsethood measure, (A1) is the only condition for normalized measure, which is to extend subsethood function. If our purpose is to measure the degree of coarse-fine relation of two granules and the boundary conditions defined in Definition \ref{def_subsethood-atoms} are the minimum requirements that subsethood measures can truthfully reflect the basic properties of inclusion degree or coarse-fine degree unless we do not consider the special case $a \cap b = \emptyset$. If our purpose is to judge whether a granule is coarser than or finer than another granule, then the axiom (A1) is enough for normalized measure, i.e. boundary condition, and the focus is on monotonicity.
\begin{definition}
	For any three atomic granules $a, b, c \in \sigma(X)$ on a universe $X$, a measure of subsethood $sh:\sigma(X) \times \sigma(X) \longrightarrow {[}0,1{]}$ is called a monotonically increasing measure if it satisfies the following monotone properties: 
	\begin{description}      
		\item [(A3)] $b \subseteq c \Rightarrow sh(b, a) \le sh(c, a)$;
		\item [(A4)] $b \subseteq c \Rightarrow sh(a,c) \le sh(a,b)$.
	\end{description}
\end{definition}

In \cite{yao2014}, Yao and Deng discussed four monotone properties of subsethood measures among three sets $a, b, c \in \sigma(X)$ as follows.
\begin{description}
	\item [(M1)] $b \subseteq c \Rightarrow sh(b,a) \le sh(c,a);$
	\item [(M2)] $b \subseteq c \wedge (b \cap a = c \cap a) \Rightarrow sh(a,c) \le sh(a,b);$
	\item [(M3)] $b \subseteq c \Rightarrow sh(a,c) \le sh(a,b);$
	\item [(M4)] $a \subseteq b \subseteq c \Rightarrow sh(a,c) \le sh(a,b);$
\end{description}

Comparing with the conditions (M1) and (M3), we know the monotonicity of function $sh(a,b)$ is reversed with that of function $sh(b,a)$, and we have (A3) $\Rightarrow$ (A4) and (A4) $\Rightarrow$ (A3). Therefore, (A3) or (A4) alone can be thought as the monotonically increasing condition of subsethood. In condition (M2), $b \cap a = c \cap a$ is the greatest lower bound of $a,b$ and $c$, which reminds us to think about its dual question, that is, their corresponding least upper bound $b \cup a = c \cup a$. Therefore, we have the following monotone properties.

\begin{description}
	\item [(A5)~] $b \subseteq c \wedge (b \cap a = c \cap a) \Rightarrow sh(b,a) \le sh(c,a)$;
	\item [(A6)~] $b \subseteq c \wedge (b \cap a = c \cap a) \Rightarrow sh(a,c) \le sh(a,b)$;
	\item [(A7)~] $b \subseteq c \wedge (b \cup a = c \cup a) \Rightarrow sh(b,a) \le sh(c,a)$;
	\item [(A8)~] $b \subseteq c \wedge (b \cup a = c \cup a) \Rightarrow sh(a,c) \le sh(a,b)$;
	\item [(A9)~] $a \subseteq b \subseteq c \Rightarrow sh(b,a) \le sh(c,a)$;
	\item [(A10)] $a \subseteq b \subseteq c \Rightarrow sh(a,c) \le sh(a,b)$;
	\item [(A11)] $b \subseteq c \subseteq a \Rightarrow sh(b,a) \le sh(c,a)$;
	\item [(A12)] $b \subseteq c \subseteq a \Rightarrow sh(a,c) \le sh(a,b)$.
\end{description}

The axioms (A5), (A7), (A9) and (A11) are weaker versions of (A3), i.e., (A3) $ \Rightarrow $ (A5), (A7), (A9) and (A11); the axioms (A6), (A8), (A10) and (A12) are weaker versions of (A4), i.e., (A4)$ \Rightarrow $ (A6), (A8), (A10) and (A12). Therefore, we can only discuss the axioms (A1), (A2), (A3) and (A4). The axioms (A5) and (A6) are the dual questions of (A7) and (A8) respectively, and the axioms (A9) and (A10) are the dual questions of (A11) and (A12) respectively. 

Yao and Deng \cite{yao2014} reviewed existing subsethood measures including $sh_l$\cite{bandler1978,bandler1980,burillo2000,bustince2006,dubois1980,fan1999,goguen1969,gomolinska2008,kosko1986, kosko1990,wang1994,willmott1980,young1996,zhangm2003}, $sh_\cap$~\cite{gomolinska2008,sanchez1979}, $sh_\cup$~\cite{fan1999,kosko1986,kuncheva1992,willmott1986}, $sh_\cap^c$~\cite{fan1999},
and $sh_\cup^c$ ~\cite{fan1999,zhangl1996,zhangl1996a} that have been considered in many studies. Most of them focus on fuzzy sets, but not on crisp sets. Yao and Deng gives the five subsethood measures of two crisp sets and have the corresponding probabilistic interpretations as follows.
\begin{align*}
	sh_1(b,a) &=sh_l (b,a)= \frac{|a^c \cup b|}{|X|}= Pr(a^c \cup b);\\
	sh_2(b,a) &= sh_{\cap}(b,a) = \frac{|a \cap b|}{|a|}=Pr(b | a);\\
	sh_3(b,a) &= sh_{\cup}(b,a) =\frac{|b|}{|a \cup b|}=Pr(b | a\cup b);\\
	sh_4(b,a) &= sh^{c}_{\cup}(b,a) =\frac{|a^c|}{|a^c \cup b^c|} = Pr(a^c | a^c \cup b^c); \\
	sh_5(b,a) &= sh^{c}_{\cap}(b,a) = \frac{|a^c \cap b^c|}{|b^c|}= Pr(a^c | b^c).
\end{align*}

If any of the value of subsethood measures is equal to 1, and we can judge the atomic $a$ is a subset of $b$. It can be seen that only $sh_{\cap}$ satisfies both (A1) and (A2). 

\begin{definition}
	For any three atomic granules $a, b, c \in \sigma(X)$, a measure of subsethood $sh:\sigma(X) \times \sigma(X) \longrightarrow {[}0,1{]}$ is called a monotonically decreasing measure if it satisfies the following monotone properties: 
	\begin{description}      
		\item [(A3$'$)] $b \subseteq c \Rightarrow sh(c,a) \le sh(b,a)$;
		\item [(A4$'$)] $b \subseteq c \Rightarrow sh(a, b) \le sh(a, c)$.
	\end{description}
\end{definition}

Then these $sh_i'(\cdot,\cdot)=1-sh_i(\cdot,\cdot) (i=1,\cdots,5)$, which can be called supsethood, are the monotonically decreasing measures corresponding to $sh_i(b,a) (i=1,\cdots,5)$, respectively, and every $sh_i'(b,a)(i=1,\cdots,5)$ can be used to define conditional fineness. For these $sh_i'(i=1,\cdots,5)$, we have
\begin{description}      
	\item [(A1$'$)] $sh_i'(b, a) = 0 \Longleftrightarrow a \subseteq b.$
\end{description}
For $sh_2'$, we also have
\begin{description}      
	\item [(A2$'$)] $sh_2'(b, a) = 1 \Longleftrightarrow a \cap b = \emptyset.$
\end{description}

\subsection{Subsethood of Two Equivalence Granules}

A subsethood measure of two sets is a quantitative generalization of the set inclusion relation, and a subsethood measure of two granules should be a quantitative generalization of the coarse-fine relation. 

\begin{definition}
	\label{def_subsethood-granules}
	For any two equivalence granules $A,B$ on $X$,
	\begin{enumerate}
		\item a function $sh(B,A) \to [0,1]$ is called a normalized measure of conditional granularity or subsethood if it satisfies the following two axioms:
		\begin{description}
			\item [(A1)] $sh(B, A) = \frac{m}{n} \Longleftrightarrow B \succeq A;$
			\item [(A2)] $sh(B, A) = 0 \Longleftrightarrow A \wedge B = \emptyset.$
		\end{description}
		\item a function $sh(B,A) \to [0,1]$ is called a normalized measure of conditional fineness or subsethood if it satisfies the following two axioms:
		\begin{description}
			\item [(A1$'$)] $sh(B, A) = 0 \Longleftrightarrow B \succeq A$;
			\item [(A2$'$)] $sh(B, A) = \frac{m}{n} \Longleftrightarrow A \wedge B = \emptyset,$
		\end{description}
	\end{enumerate}
	where $n$ is the cardinality of $X$ and $m$ is the smaller one of the cardinalities of the sets $A$ and $B$.
\end{definition}

The monotonically increasing and monotonically decreasing measures corresponding to conditional granularity and conditional fineness respectively can be defined as follows.

\begin{definition}
	For any three equivalence granules $A, B, C$ on $X$, 
	\begin{enumerate}
		\item a measure of subsethood $sh:\Pi(\sigma(X)) \times \Pi(\sigma(X)) \longrightarrow {[}0,1{]}$ is called a monotonically increasing measure if it satisfies the following monotone properties: 
		\begin{description}      
			\item [(A3)] $C \succeq B \Rightarrow sh(B, A) \le sh(C, A)$;
			\item [(A4)] $C \succeq B \Rightarrow sh(A,C) \le sh(A,B)$.
		\end{description}
		\item a measure of subsethood $sh:\Pi(\sigma(X)) \times \Pi(\sigma(X)) \longrightarrow {[}0,1{]}$ is called a monotonically decreasing measure if it satisfies the following monotone properties:
		\begin{description}      	
			\item [(A3$'$)] $C \succeq B \Rightarrow sh(C,A) \le sh(B,A)$;
			\item [(A4$'$)] $C \succeq B \Rightarrow sh(A, B) \le sh(A, C)$.
		\end{description}
	\end{enumerate}
\end{definition}

We also have (A3) $\Rightarrow$ (A4) and (A4) $\Rightarrow$ (A3), and (A3$'$) $\Rightarrow$ (A4$'$) and (A4$'$) $\Rightarrow$ (A3$'$). Therefore, (A3) or (A4) alone can be the monotonically increasing condition, and (A3$'$) or (A4$'$) alone can be the monotonically decreasing condition.

For any equivalence granules $A,B,C$ on $X$, the conditions (A5), $\cdots$, (A12) and the conditions (A5$'$), $\cdots$,(A12$'$) are as follows.

\begin{description}
	\item [(A5)] $C \succeq B \wedge (B \wedge A = C \wedge A) \Rightarrow sh(B,A) \le sh(C,A)$;
	\item [(A6)] $C \succeq B \wedge (B \wedge A = C \wedge A) \Rightarrow sh(A,C) \le sh(A,B)$;
	\item [(A7)] $C \succeq B \wedge (B \vee A = C \vee A) \Rightarrow sh(B,A) \le sh(C,A)$;
	\item [(A8)] $C \succeq B \wedge (B \vee A = C \vee A) \Rightarrow sh(A,C) \le sh(A,B)$;
	\item [(A9)] $C \succeq B \succeq A \Rightarrow sh(B,A) \le sh(C,A)$;
	\item [(A10)] $C \succeq B \succeq A \Rightarrow sh(A,C) \le sh(A,B)$;
	\item [(A11)] $A \succeq C \succeq B \Rightarrow sh(B,A) \le sh(C,A)$;
	\item [(A12)] $A \succeq C \succeq B \Rightarrow sh(A,C) \le sh(A,B)$;
	\item [(A5$'$)] $C \succeq B \wedge (B \wedge A = C \wedge A) \Rightarrow sh(C,A) \le sh(B,A)$;
	\item [(A6$'$)] $C \succeq B \wedge (B \wedge A = C \wedge A) \Rightarrow sh(A,B) \le sh(A,C)$;
	\item [(A7$'$)] $C \succeq B \wedge (B \vee A = C \vee A) \Rightarrow sh(C,A) \le sh(B,A)$;
	\item [(A8$'$)] $C \succeq B \wedge (B \vee A = C \vee A) \Rightarrow sh(A,B) \le sh(A,C)$;
	\item [(A9$'$)] $C \succeq B \succeq A \Rightarrow sh(C,A) \le sh(B,A)$;
	\item [(A10$'$)] $C \succeq B \succeq A \Rightarrow sh(A,B) \le sh(A,C)$;
	\item [(A11$'$)] $A \succeq C \succeq B \Rightarrow sh(C,A) \le sh(B,A)$;
	\item [(A12$'$)] $A \succeq C \succeq B \Rightarrow sh(A,B) \le sh(A,C)$.
\end{description}

The conditions (A5), (A7), (A9) and (A11) are weaker versions of (A3), i.e., (A3) $ \Rightarrow $ (A5), (A7), (A9) and (A11); the axioms (A6), (A8), (A10) and (A12) are weaker versions of (A4), i.e., (A4) $ \Rightarrow $ (A6), (A8), (A10) and (A12). The conditions (A5), (A6), (A7) and (A8) are a special case of (A9), (A10), (A11) and (A12), respectively. The conditions (A5) and (A6) are the dual questions of (A7) and (A8) respectively, and the conditions (A9) and (A10) are the dual questions of (A11) and (A12) respectively. While the axiom (Ai) is reversed with (Ai$'$) ($i=1,\cdots,12$). The first four are their basic properties. 

Given two equivalence granules $A=\{a_1,\cdots, a_k\}$ and $B=\{b_1,\cdots, b_l\}$ on $X$. Then there are $|a_i \cap b_j|(a_i \cap b_j)(i=1,\cdots,k, j=1,\cdots,l)$ in $A \wedge B$, where $|a_i \cap b_j|$ is the cardinality of $a_i \cap b_j$. We can normalize these $|a_i \cap b_j|(i=1,\cdots,k,j=1,\cdots,l)$ and get a probability distribution which is called a probability distribution of the granule $A \wedge B$ denoted as $P_{A\wedge B}$.
\begin{align}
	\label{prob-distribution}
	P_{A\wedge B} &= \left( {p(a_1 \cap b_1), \cdots, p(a_i \cap b_j),\cdots, p(a_k \cap b_l)} \right) \nonumber \\
	&= \left( {\frac{|a_1 \cap b_1|}
		{|X|}, \cdots ,\frac{|a_i \cap b_j|}
		{|X|}, \cdots, \frac{|a_k \cap b_l|}
		{|X|}} \right),
\end{align}
where $p(a_i \cap b_j)$ indicates the probability of the intersection of $a_i$ and $b_j$ contained in $X$. We have the following result.

\begin{theorem}
	\[\sum\limits_{i = 1}^k\sum\limits_{j = 1}^l {p(a_i \cap b_j)} \le \frac{m}{n},\]
	where $n$ is the cardinality of the universe $X$ and $m$ is the smaller one of the cardinalities of the sets $A$ and $B$.
\end{theorem}

\begin{proof}
	Let us assume that the cardinality of $A$ is the smaller one and $|A|=m$, then, we have
	\begin{align*}
		\sum\limits_{i = 1}^k\sum\limits_{j = 1}^l {p(a_i \cap b_j)}&=\sum\limits_{i = 1}^k{\frac{1}{|X|}(|a_i \cap b_1|+\cdots + |a_i \cap b_l|)}  \\
		&=\frac{1}{n}\sum\limits_{i = 1}^k{|a_i \cap (b_1\cup \cdots \cup b_l)|} \\
		&= \frac{1}{n}\sum\limits_{i = 1}^k{|a_i \cap B|} \\
		&\le \frac{1}{n}\sum\limits_{i = 1}^k{|a_i|}=\frac{m}{n}. 
	\end{align*}
\end{proof} 

Given two equivalence granules $A=\{a_1,\cdots, a_k\}$ and $B=\{b_1,\cdots, b_l\}$ on $X$. Then, for each $sh_m(m=1,\cdots,5)$, the conditional granularity of $B$ with respect to $A$ is defined by the expectations of $sh_m (m=1,\cdots,5)$ with respect to the probability distribution of $A \wedge B$.
\begin{definition}
	\label{def_conditional-granularity-subsethood}
	\begin{align}
		G_{m}(B|A)&= sh_m(B,A)=E_{P_{A\wedge B}}({sh_m}(\cdot,\cdot))\nonumber\\
		&= \sum\limits_{i = 1}^k\sum\limits_{j = 1}^l {p(a_i \cap b_j)sh_m(b_i, a_i)}. 
	\end{align}
\end{definition}

In general, we can take $sh_m'(\cdot,\cdot)= 1-sh_m(\cdot,\cdot)(m=1,\cdots,5)$. Then, the expectations of $sh_m'(\cdot,\cdot)= 1-sh_m(\cdot,\cdot)(m=1,\cdots,5)$ with respect to the probability distribution of $A \wedge B$ is $E_{P_{A\wedge B}}({sh_m'}(\cdot,\cdot))$
\begin{align}
	&= \sum\limits_{i = 1}^k\sum\limits_{j = 1}^l {p(a_i \cap b_j)sh_m'(b_j, a_i)} \nonumber \\
	&= \sum\limits_{i = 1}^k\sum\limits_{j = 1}^l {p(a_i \cap b_j)(1-sh_m(b_j, a_i)} )\nonumber \\
	&= \sum\limits_{i = 1}^k\sum\limits_{j = 1}^l {p(a_i \cap b_j)}-\sum\limits_{i = 1}^k\sum\limits_{j = 1}^l {p(a_i \cap b_j){sh_m(b_j, a_i)} }\nonumber \\
	&\le \frac{m}{n}-G_{m}(B|A).
\end{align}
Given two equivalence granules $A=\{a_1,\cdots, a_k\}$ and $B=\{b_1,\cdots, b_l\}$ on $X$. The conditional fineness of $B$ with respect to $A$ can be defined by 
\begin{definition}
	\label{def_conditional-fineness-subsethood}
	\begin{enumerate}
		\item[]
		\item $F_{i}(B|A) = \frac{m}{n}-G_{i}(B|A)(i=1,\cdots,5);$
	\end{enumerate}
\end{definition}

By the above definition, we can easily get the following theorems.

\begin{theorem}
	For any equivalence granules $A$ and $B$ on $X$, we have
	\begin{enumerate}
		\item $G_{i}(B|A)(i=1,\cdots,5)$ satisfies the axiom (A2), namely,\\ $G_{i}(B|A)(i=1,\cdots,5)= 0 \Longleftrightarrow A \wedge B = \emptyset$;
		\item $F_{i}(B|A)(i=1,\cdots,5)$ satisfies the axiom (A2$'$), namely,\\ $F_{i}(B|A)(i=1,\cdots,5)= \frac{m}{n} \Longleftrightarrow A \wedge B = \emptyset.$
	\end{enumerate}
\end{theorem} 

\begin{theorem}
	For any equivalence granules $A$ and $B$ on $X$, we have 
	\begin{enumerate}
		\item $0 \le G_{i}(B|A) \le 1(i=1,\cdots,5)$;
		\item $0 \le F_{i}(B|A) \le 1(i=1,\cdots,5).$
	\end{enumerate}
\end{theorem} 

\begin{theorem}
	For any equivalence granule $A$ on $X$, we have 
	\begin{enumerate}
		\item $G_{i}(A|\{X\}) = G_{i}(A)(i=1,\cdots,5)$;
		\item $F_{i}(A|\{X\}) = F_{i}(A)(i=1,\cdots,5).$
	\end{enumerate}
\end{theorem}

\begin{definition}
	Given two equivalence granules $A=\{a_1,\cdots, a_k\}$ and $B=\{b_1,\cdots, b_l\}$ on $X$. For any  $i,j(i=1,\cdots,k,j=1,\cdots,l)$, we have all $p(a_i \cap b_j) = 0$, then $A$ and $B$ are independent, and, particularly, $B$ is called the quotient complement of $A$ if $B$ has only one atomic granule. 
\end{definition}

\begin{theorem}
	\label{thm_independent}
	For any two equivalence granules $A$ and $B$ on $X$, we have
	\begin{enumerate}
		\item $A$ and $B$ is independent if and only if $G_{i}(B|A) = G_{i}(A|B) = 0(i=1,\cdots,5)$;
		\item $A$ and $B$ is independent if and only if $F_{i}(B|A) = F_{i}(A|B) = \frac{m}{n}(i=1,\cdots,5),$
	\end{enumerate}
	where $n$ is the cardinality of the universe $X$ and $m$ is the smaller one of the cardinalities of the sets $A$ and $B$.
\end{theorem}

Now we start to prove $G_{i}(B|A)(i=1,\cdots,5)$ satisfies the axiom (A1), (A3) and (A4), and $F_{i}(B|A)(i=1,\cdots,5)$ satisfies the axiom (A1$'$), (A3$'$) and (A4$'$). 

\begin{theorem}
	\label{thm_axiom1}
	Assume that $A = \{a_1, \cdots, a_k \}$ and $B = \{b_1, \cdots, b_l \}$ are two equivalence granules on $X$. Then 
	\begin{enumerate}
		\item $A$ is finer than $B$ if and only if $G_{i}(B|A)= \frac{m}{n}(i=1,\cdots,5)$;
		\item $A$ is finer than $B$ if and only if $F_{i}(B|A)=1- \frac{m}{n}(i=1,\cdots,5)$,
	\end{enumerate}
	where $n$ is the cardinality of the universe $X$ and $m$ is the smaller one of the cardinalities of the sets $A$ and $B$.
\end{theorem}

The proofs are seen in Appendix \ref{appendix:axiom1}. By the above theorem, we can get the following corollary.

\begin{corollary}
	\label{corollory_axiom1}
	Assume that $A = \{a_1, \cdots, a_k \}$ and $B = \{b_1, \cdots, b_l \}$ are two quotient granules on $X$. Then 
	\begin{enumerate}
		\item $A$ is finer than $B$ if and only if $G_{i}(B|A)= 1(i=1,\cdots,5)$;
		\item $A$ is finer than $B$ if and only if $F_{i}(B|A)=0(i=1,\cdots,5)$.
	\end{enumerate}
\end{corollary}

\begin{lemma}
	\label{lem_subsethood-granularity}
	For any two equivalence granules $B = \{b_1,\cdots,b_{l+1}\}$ and $C = \{c_1,\cdots, c_l\}$ on $X$. If $ b_l  \cup b_{l+1}  \subseteq c_l ,b_i  = c_i (i = 1, \cdots ,l-1),$ then for any equivalent granule $A = \{a_1, \cdots, a_k\}$ on $X$, we have
	\begin{enumerate}
		\item $G_{i}(B|A) \le G_{i}(C|A);$
		\item $G_{i}(A|C) \le G_{i}(A|B).$
	\end{enumerate}
\end{lemma}

The proofs are seen in Appendix \ref{appendix:lemma}. Accordingly, we have the following result.

\begin{lemma}
	\label{lem_subsethood-fineness}
	For any two equivalence granules $B = \{b_1,\cdots,b_{l+1}\}$ and $C = \{c_1,\cdots, c_l\}$ on $X$. If $ b_l  \cup b_{l+1}  \subseteq c_l ,b_i  = c_i (i = 1, \cdots ,l-1),$ then, for any equivalent granule $A = \{a_1, \cdots, a_k\}$ on $X$, we have
	\begin{enumerate}
		\item $F_{i}(A|C) \le F_{i}(A|B);$
		\item $F_{i}(B|A) \le F_{i}(A|C).$
	\end{enumerate}
\end{lemma}

For any two equivalence granules $B= \{b_1,\cdots, b_l \}$ and $C= \{c_1,\cdots, c_m \}$ on $X$ and $B$ is finer than $C$. For any $c_j$ in $C$, there are two cases: either there exists some $b_i$ subjecting to $b_i = c_j$ or there exist some $b_i$s which satisfy that the union of these $b_i$ is equal to $c_j$. By repeating the above Lemmas, we can easily get the following two theorems.

\begin{theorem}
	\label{thm_subsethood-axiom3&4}
	For any three equivalence granules $A = \{a_1,\cdots, a_k \}, B = \{b_1,\cdots, b_l \}$ and $C = \{c_1,\cdots, c_m \}$ on $X$, we have, for $i=1,\cdots,5$,
	\begin{description}
		\item [(A3)] $C \succeq B \Rightarrow G_{i}(B|A) \le G_{i}(C|A)$;
		\item [(A4)] $C \succeq B \Rightarrow G_{i}(A|C) \le G_{i}(A|B)$.
	\end{description}
\end{theorem}

\begin{theorem}
	\label{thm_subsethood-axiom3'&4'}
	For any three equivalence granules $A = \{a_1,\cdots, a_k \}, B = \{b_1,\cdots, b_l \}$ and $C = \{c_1,\cdots, c_m \}$ on $X$, we have, for $i=1,\cdots,5$,
	\begin{description}
		\item [(A3$'$)] $C \succeq B \Rightarrow F_{i}(C|A) \le F_{i}(B|A)$;
		\item [(A4$'$)] $C \succeq B \Rightarrow F_{i}(A|B) \le F_{i}(A|C)$.
	\end{description}
\end{theorem}

All $sh_i(i=1,\cdots,5)$ satisfy the axioms (A1), (A2), (A3) and (A4) or the axioms (A1$'$), (A2$'$), (A3$'$) and (A4$'$). The axioms (A1) and (A2) (or (A1$'$) and (A2$'$)) are the two normalized boundary conditions. However, there does not exist the special case $A \wedge B = \emptyset$ when the granules are in a complete information system or subsystem. Therefore, it is reasonable to think (A1) (or (A1$'$)) as the normalized boundary condition.  The axioms (A3) and (A4) (or (A3$'$) and (A4$'$)) are monotone conditions which can be replaced by their weak axioms (A5) and (A6) (or (A7) and (A8) or (A9) and (A10) or (A11) and (A12)) or the axioms (A5$'$) and (A6$'$) (or (A7$'$) and (A8$'$) or (A9$'$) and (A10$'$) or (A11$'$) and (A12$'$)), and any one of monotone conditions alone can also be regarded as the monotone condition because they imply each other. Thus, the boundary condition (A1) and any one of the monotone conditions constitute the basic axioms. 

\subsection{Subsethood Entropy}

Entropy, an important concept of thermodynamics, was introduced by German physicist Rudolph Clausius in 1865 \cite{clausius1865}. The term of entropy has been used in various areas like chemistry, physics, biology, cosmology, economics, statistics, sociology, weather science, and information science. 
Information entropy as a concept was introduced by C. E. Shannon who was the founder of information theory in 1948 \cite{shannon1948}. Information entropy was introduced to measure the granularity of each partition ~ \cite{beaubouef1998,duntsch1998, duntsch2001, klir1988,lee1987,liang2002,liang2004,liang2006,liang2009,miao1998,miao1999,qian2008,qian2009,wang2008,wierman1999,yao2003a,yao2012b,zhu2012}. After that, many other entropies have been introduced, and Hartley entropy, collision entropy, R{\'e}nyi entropy, and min-entropy. have been introduced to measure granularity or fineness of equivalence granules. Accordingly, the subsethood measures $sh_i (i=1,\cdots, 5)$ can also be generalized to their corresponding subsethood entropies by the probability distribution of the meet of two granules in Equation (\ref{prob-distribution}). 

Assume $A=\{a_1,\cdots, a_k\}$ and $B=\{b_1,\cdots, b_l\}$ are two equivalence granules on $X$. For each $sh_i(i=1,\cdots,5)$, its corresponding subsethood entropy can be defined by.

\begin{definition}
	\label{def_conditional-fineness-entropy}
	\begin{align}
		H'_{i}(B|A)&= H'_{sh_i}(B|A)=E_{P_{A\wedge B}}(\log{sh_i}(\cdot,\cdot))\nonumber \\
		&= -\sum\limits_{i = 1}^k\sum\limits_{j = 1}^l {p(a_i \cap b_j)\log{sh_i(b_j, a_i)}}.
	\end{align}
\end{definition}

$H'_{i}(B|A)$ is a monotonically decreasing function, and it is also called the conditional fineness entropy of $B$ with respect to $A$. Then, the expectations of logarithm of $\log{sh_i'(\cdot,\cdot)}= \log{nsh_i(\cdot,\cdot)}(i=1,\cdots,5)$ with respect to the probability distribution of $A \wedge B$ is $E_{P_{A\wedge B}}({sh_i'}(\cdot,\cdot))$
\begin{align}
	&= \sum\limits_{i = 1}^k\sum\limits_{j = 1}^l {p(a_i \cap b_j)\log{sh_i'(b_j, a_i)}} \nonumber \\
	&= \sum\limits_{i = 1}^k\sum\limits_{j = 1}^l {p(a_i \cap b_j)(\log{n}+\text{log}{sh_i(b_j, a_i)}} )\nonumber \\
	&= \log{n}\sum\limits_{i = 1}^k\sum\limits_{j = 1}^l {p(a_i \cap b_j)}+\sum\limits_{i = 1}^k\sum\limits_{j = 1}^l {p(a_i \cap b_j){\text{log}{sh_i(b_j, a_i)} }}\nonumber \\
	&\le \frac{m}{n}\log{n}-H_{i}(B|A).
\end{align}

Therefore, for any two equivalence granules $A$ and $B$ on $X$, the conditional granularity entropy of $B$ with respect to $A$ can also be defined by
\begin{definition}
	\label{def_conditional-granularity-entropy}
	\begin{align*}
		H_{i}(B|A)&=H_{sh_i}(B|A) \\
		&= \frac{m}{n}\log{n}-H'_{i}(B|A)(i=1,\cdots,5).
	\end{align*}
\end{definition}

In those conditional granularities and conditional finenesses, for any equivalence granule $A$ on $X$, we have $G(A|\{X\})=G(A)$ and $F(A|\{X\})=F(A)$, and thus we can define

\begin{definition}
	\begin{enumerate} 
		\item[]
		\item $H_{i}(A) = H_{i}(A|\{X\}) (i=1,\cdots,5)$;
		\item $H_{i}'(A) = H'_{i}(A|\{X\}) (i=1,\cdots,5).$
	\end{enumerate}
\end{definition}

By the above definitions, we can easily get the following theorems.
\begin{theorem}
	For any two equivalence granules $A$ and $B$ on $X$, we have 
	\begin{enumerate} 
		\item $0 \le H_{i}(B|A) \le \log{n}(i=1,\cdots,5)$;
		\item $0 \le H'_{i}(B|A) \le \log{n}(i=1,\cdots,5).$
	\end{enumerate}
\end{theorem}

\begin{theorem}
	Given a universe $X$. For any two granules $A = \{a_1, \cdots, a_k\}$ and $B = \{b_1, \cdots, b_l\}$ on $X$,  we have
	\begin{enumerate}
		\item $B \succeq A  \Rightarrow H'_i(B|A) =0$;
		\item if $H'_i(B|A) =0$, then $A$ is finer than $B$ or $A$ and $B$ are independent.
	\end{enumerate}
\end{theorem}
\begin{proof}
	\begin{enumerate}
		\item If $A$ is finer than $B$, that is, for any $a_i(i=1,\cdots,k)$, there exists only one $b_j(j \in \{1,\cdots,l\})$, which subjects $a_i \subseteq b_j$. That is, $\log{sh( b_j, a_i)}=0$ because all $sh_i(b,a)(i=1,\cdots,5)$ reach the maximum 1 when $ a \subseteq  b$, i.e., $a  \cap b  = a$. For other $h \ne j \in \{1,\cdots,l\}$, we have $p(a_i \cap b_j) = 0$. Therefore, $p(a_i \cap b_j)\log{sh( b_j,  a_i)} = 0(i = 1, \cdots ,k,j = 1, \cdots ,l)$. Thus $H'_i(B|A) = 0$.
		\item Every item of $- \sum\limits_{i = 1}^k {\sum\limits_{j = 1}^l {p(a_i \cap b_j)\log{sh( b_j, a_i)}}}$ is more than or equal to $0$ if $H'_i(B|A)=0$, and thus, we have $p(a_i \cap b_j)\log{sh(b_j, a_i)}= 0 (i=1,\cdots,k,j=1,\cdots,l)$. There are two cases:\\
		For any  $i,j(i=1,\cdots,k,j=1,\cdots,l)$, all $p(a_i \cap b_j) = 0$, that is, $A$ and $B$ are independent; \\
		For each $a_i(i \in \{1,\cdots,k\})$, for $j=1,\cdots,l)$, either $p(a_i \cap b_j) = 0$ or $sh(b_j, a_i) = 1$, that is, $|{a_i} \cap {b_j}| = 0$ or $a_i \subseteq b_j$. By the definitions of equivalence granules, for each $a_i(i \in \{1,\cdots,k\})$, there exists only one $j \in \{1,\cdots,l\}$ which subjects to $a_i \subseteq b_j$, and so $A$ is finer than $B$. 
	\end{enumerate}
\end{proof}

\begin{theorem}
	Given a universe $X$. For any two equivalence granules $A = \{a_1, \cdots, a_k\}$ and $B = \{b_1, \cdots, b_l\}$ on $X$,  we have
	\begin{enumerate}
		\item $B \succeq A  \Rightarrow H_i(B|A) =\frac{m}{n}\log{n}$;
		\item if $H_i(B|A) =\frac{m}{n}\log{n}$, then $A$ is finer than $B$ or $A$ and $B$ are independent,
	\end{enumerate}
	\noindent	where $n$ is the cardinality of $X$ and $m$ is the smaller one of the cardinalities of $A$ and $B$.
\end{theorem}

It can be seen that $H_{i}(B|A)$ does not satisfy axiom (A1) and $H'_{i}(B|A)$ does not satisfy axiom A1$'$ even if they are normalized. For any two equivalence granules $A = \{a_1, \cdots, a_k\}$ and $B = \{b_1, \cdots, b_l\}$ in a complete information system on $X$, and we have the following result.

\begin{corollary}
	\begin{enumerate}
		\item[]
		\item $B \succeq A  \Longleftrightarrow H_i(B|A) =\frac{m}{n}\log{n}$;
		\item $B \succeq A  \Longleftrightarrow H'_i(B|A) =0$,
	\end{enumerate}
	\noindent	where $n$ is the cardinality of $X$ and $m$ is the smaller one of the cardinalities of $A$ and $B$.
\end{corollary}

That means $H_{i}(B|A)$ satisfies axiom (A1) and $H'_{i}(B|A)$ satisfies axiom A1$'$ if they are normalized. However, $H_{i}(B|A)$ does not satisfy axiom (A2) and $H'_{i}(B|A)$ does not satisfy axiom A2$'$.

\begin{corollary}
	Assume that $A = \{a_1, \cdots, a_k \}$ and $B = \{b_1, \cdots, b_l \}$ are two quotient granules on $X$. Then 
	\begin{enumerate} 
		\item $A$ is finer than $B$ if and only if $H_{i}(B|A)= \log{n}(i=1,\cdots,5)$;
		\item $A$ is finer than $B$ if and only if $H'_{i}(B|A)=0(i=1,\cdots,5)$.
	\end{enumerate}
\end{corollary}

Because $H_{i}$ and $H'_i$ keep the same monotonicity of $G_i$ and $F_i$ respectively, we have the following result.

\begin{lemma}
	\label{lem_subsethood-granularity}
	For any two equivalence granules $B = \{b_1,\cdots,b_{l+1}\}$ and $C = \{c_1,\cdots, c_l\}$ on $X$. If $ b_l  \cup b_{l+1}  \subseteq c_l ,b_i  = c_i (i = 1, \cdots ,l-1),$ then for any equivalent granule $A = \{a_1, \cdots, a_k\}$ on $X$, we have
	\begin{enumerate}
		\item $H_{i}(B|A) \le H_{i}(C|A)$ and $H_{i}(A|C) \le H_{i}(A|B);$
		\item $H'_{i}(C|A) \le H'_{i}(B|A)$ and $H'_{i}(A|B) \le H'_{i}(A|C).$
	\end{enumerate}
\end{lemma}

For any two equivalence granules $B= \{b_1,\cdots, b_l \}$ and $C= \{c_1,\cdots, c_m \}$ on $X$ and $B$ is finer than $C$. For any $c_j$ in $C$, there are two cases: either there exists some $b_i$ subjecting to $b_i = c_j$ or there exist some $b_i$s which satisfy that the union of these $b_i$ is equal to $c_j$. By repeated use of above Lemma, we can easily get the following two theorems.

\begin{theorem}
	\label{thm_entropy-axiom3&4}
	For any three equivalence granules $A = \{a_1,\cdots, a_k \}, B = \{b_1,\cdots, b_l \}$ and $C = \{c_1,\cdots, c_m \}$ on $X$, we have, for $i=1,\cdots,5$,
	\begin{description}
		\item [(A3)] $C \succeq B \Rightarrow H_{i}(B|A) \le H_{i}(C|A)$;
		\item [(A4)] $C \succeq B \Rightarrow H_{i}(A|C) \le H_{i}(A|B)$.
	\end{description}
\end{theorem}

\begin{theorem}
	\label{thm_entropy-axiom3'&4'}
	For any three equivalence granules $A = \{a_1,\cdots, a_k \}, B = \{b_1,\cdots, b_l \}$ and $C = \{c_1,\cdots, c_m \}$ on $X$, we have, for $i=1,\cdots,5$,
	\begin{description}
		\item [(A3$'$)] $C \succeq B \Rightarrow H'_{i}(C|A) \le H'_{i}(B|A)$;
		\item [(A4$'$)] $C \succeq B \Rightarrow H'_{i}(A|B) \le H'_{i}(A|C)$.
	\end{description}
\end{theorem}

\section{Conclusion}

GrC is to imitate two types of granulation process in human recognition: micro granular analysis process and macro granular analysis processes. Micro granular analysis focuses on the parts while macro granular analysis focuses on the whole. All the knowledge generated in the process of micro granular analysis constitute a micro knowledge space, and all the knowledge generated in the process of macro granular analysis constitute a macro knowledge space. Viewing an information system from micro perspective, we can get a micro knowledge space, and, viewing it from macro perspective, we can get a macro knowledge space, from which we obtain the rough set model and the spatial rough granule model respectively. The classical rough set model can only be used for complete information systems, while the rough set model obtained from micro knowledge space can also be used for incomplete information systems, what's more, the universe of discourse can be any domain. The spatial rough granule model will play a pivotal role in the problem solving of structures like graph partition, image processing, face recognition, 3D technologies, etc. 

Subsethood measures have been well studied and generally accepted in many fields other than fuzzy sets and rough sets. Subsethood measures which is used to measure the set-inclusion relation between two sets are generalized to measure the coarse-fine relation between two granules. This paper defines conditional granularity, conditional fineness, conditional granularity entropy and conditional fineness entropy and discuss their properties including coarse-fine relation determination theorem, and all of these are very important foundations for learning and reasoning of structural problems. These measures can be used for fuzzy granules, and they have a close relation with similarity and difference, which will be studied in the future. 

\section*{Appendix}
{\appendices
\section*{Proof of the Theorem \ref{thm_axiom1}}
\label{appendix:axiom1}
	We only prove $G_{1}(B|A)=\frac{m}{n}$, the others are similar
	\begin{proof}
		The sufficiency is obvious. Now we prove its necessity. We may assume that $|A| =m  \le |B|$.
		By Definition \ref{def_conditional-granularity-subsethood}, we have ${sh_l}(B,A) $
		\begin{align*}
			&= \sum\limits_{i = 1}^k {\sum\limits_{j = 1}^l {\frac{{|{a_i} \cap {b_j}|}}
					{{|X|}}} }  \times \frac{{|{a_i^c} \cup {b_j}|}}
			{{|X|}}\\
			&= \sum\limits_{i = 1}^k {\frac{1}
				{{|X{|^2}}}} \sum\limits_{j = 1}^l {|{a_i} \cap {b_j}|(|X - {a_i}| + |{a_i} \cap {b_j}|)}\\
			&= \sum\limits_{i = 1}^k {\frac{1}
				{{|X{|^2}}}} \left( {\sum\limits_{j = 1}^l {|{a_i} \cap {b_j}||X - {a_i}| + \sum\limits_{j = 1}^l {|{a_i} \cap {b_j}{|^2}} } } \right)
		\end{align*}
		Assume the union of all $b_j$ is the set $B$. Then $|a_i \cap b_1| + \cdots +  |a_i \cap b_l| = |a_i \cap (b_1 \cup \cdots \cup b_l)| = |a_i \cap B| = |a_i |$. Thus
		\begin{align*}
			&\sum\limits_{i = 1}^k {\frac{1}
				{{|X{|^2}}}} \left( {\sum\limits_{j = 1}^l {|{a_i} \cap {b_j}||X - {a_i}| + \sum\limits_{j = 1}^l {|{a_i} \cap {b_j}{|^2}} } } \right) \\
			&\le \sum\limits_{i = 1}^k {\frac{{(|{a_i}||X - {a_i}| + |{a_i}{|^2})}}
				{{|X{|^2}}}} \\
			&= \sum\limits_{i = 1}^k {\frac{{|{a_i}|}}
				{{|X|}}} \frac{{|X - {a_i}| + |{a_i}|}}
			{{|X|}} \\
			&= \sum\limits_{i = 1}^k {\frac{{|{a_i}|}}
				{{|X|}}}  = \frac{m}{n}
		\end{align*}
		When there exists some one such that $|a_i \cap b_h| = |a_i|, |a_i \cap b_j| = 0 (j  \ne h, j \in I),  \sum\nolimits_j {|a_i  \cap b_j |^2 } = |a_i|^2$ reaches the maximum, that is, for any $a_i$ there must exist some $b_h$ which satisfies $a_i \cap b_h = a_i$ and $a_i \cap b_j = \emptyset(j  \ne h, j \in I)$. Therefore $A$ is finer than $B$.
	\end{proof}
\section*{Proof of the Lemma \ref{lem_subsethood-granularity}}
\label{appendix:lemma}
	We only prove $sh_1$, and the others are similar
	\begin{proof} Here only prove (2)\\
		Suppose there are $h(0 \leq h \leq |c_l|)$ equivalence classes intersecting with $c_l$ in $A$. When $h=0$  we have
		\begin{align*}
			{sh_l}(A,B)&=\sum\limits_{i = 1}^{l + 1} {\frac{1}
				{{|X{|^2}}}} \sum\limits_{j = 1}^k {|{b_i} \cap {a_j}||b_i^c \cup {a_j}|} \\
			&= \sum\limits_{i = 1}^{l - 1} {\frac{1}
				{{|X{|^2}}}} \sum\limits_{j = 1}^k {|{b_i} \cap {a_j}||b_i^c \cup {a_j}|} \\
			&=  \sum\limits_{i = 1}^{l - 1} {\frac{1}
				{{|X{|^2}}}} \sum\limits_{j = 1}^k {|{c_i} \cap {a_j}||c_i^c \cup {a_j}|} = {sh_l}(A,C)
		\end{align*}
		When  $1 \leq h \leq |c_l|$, let them be $a_1, \cdots, a_h$, respectively, we have
		\begin{align*}
			{sh_l}(A,B) &= \sum\limits_{i = 1}^{l + 1} {\frac{1}
				{{|X{|^2}}}} \sum\limits_{j = 1}^k {|{b_i} \cap {a_j}||b_i^c \cup {a_j}|} \\
			&= \frac{{\sum\limits_{j = 1}^h {\left( {|{b_l} \cap {a_j}||b_l^c \cup {a_j}| + |{b_{l + 1}} \cap {a_j}||b_{l + 1}^c \cup {a_j}|} \right)} }}
			{{|X{|^2}}} \\
			&+ \frac{{\sum\limits_{i = 1}^{l - 1} {\sum\limits_{j = 1}^k {|{b_i} \cap {a_j}||b_i^c \cup {a_j}|} } }}
			{{|X{|^2}}} \\
			{sh_l}(A,C) &= \sum\limits_{i = 1}^{l} {\frac{1}
				{{|X{|^2}}}} \sum\limits_{j = 1}^k {|{c_i} \cap {a_j}||c_i^c \cup {a_j}|} \\
			&= \frac{{\sum\limits_{j = 1}^h {|{c_l} \cap {a_j}||c_l^c \cup {a_j}|} }}
			{{|X{|^2}}} + \frac{{\sum\limits_{i = 1}^{l - 1} {\sum\limits_{j = 1}^k {|{c_i} \cap {a_j}||c_i^c \cup {a_j}|} } }}
			{{|X{|^2}}}
		\end{align*}
		While $b_i  = c_i (i = 1, \cdots ,l-1)$, and $|{c_l} \cap {a_j}||c_l^c \cup {a_j}| $
		\begin{align*}
			&= {|({b_l} \cup {b_{l + 1}}) \cap {a_j}||{{({b_l} \cup {b_{l + 1}})}^c} \cup {a_j}|} \\
			&= {(|{b_l} \cap {a_j}| + |{b_{l + 1}} \cap {a_j}|)|(b_l^c \cup {a_j}) \cap (b_{l + 1}^c \cup {a_j})|} \\
			&\le |{b_l} \cap {a_j}||b_l^c \cup {a_j}| + |{b_{l + 1}} \cap {a_j}||b_{l + 1}^c \cup {a_j}|
		\end{align*}
		Thus we have ${sh_l}(A,C) \le {sh_l}(A,B)$.
	\end{proof}		
}

\section*{Acknowledgments}

This work is partially supported by a Discovery Grant from NSERC Canada.

\bibliographystyle{IEEEtran}
\bibliography{IEEEabrv,mybib}
\end{document}